\documentclass{article} 

\usepackage{mathrsfs,dsfont}
\usepackage{amsmath,amsfonts,graphicx,amsthm,amssymb}  
\usepackage{geometry}[a4paper,portrait,left=3.5cm,right=3.5cm,top=3.5cm,bottom=3.5cm]

\title{Analysing  heavy-tail properties of Stochastic Gradient Descent by means of Stochastic Recurrence Equations\ 
    \thanks{E.D. was supported by NCN grant 2019/33/B/ST1/00207. S.M. was supported by DFG grant ME 4473/2-1.}} 

\author{%
  Ewa~Damek\footnote{University of Wroc\l aw, Poland.
    {ewa.damek@math.uni.wroc.pl}}
  \and 
  Sebastian~Mentemeier\footnote{University of Hildesheim,
    Germany. {mentemeier@uni-hildesheim.de} }}

\DeclareMathOperator{\Id}{\mathrm{I}}
\DeclareMathOperator{\Rd}{\mathds{R}^d}
\DeclareMathOperator{\N}{\mathds{N}}
\DeclareMathOperator{\E}{\mathds{E}}
\DeclareMathOperator{\Cf}{\mathcal{C}}
\renewcommand{\S}{S}
\DeclareMathOperator{\Ps}{P^s}
\DeclareMathOperator{\Pst}{P^s_*}

\renewcommand{\P}{\mathbb{P}}

\newcommand{\nust}[1][s]{\nu^*_{#1}}
\newcommand{\nus}[1][s]{\nu_{#1}}
\newcommand{\es}[1][s]{r_{#1}}

\newcommand{\Mset}[1]{\mathrm{GL}(#1,\mathds{R})}
\newcommand{\Sym}[1][d]{\mathrm{Sym}(#1,\mathds{R})}
\newcommand{\supp}{\mathrm{supp}}

\newcommand{\interior}[1]{\mathrm{int}(#1)}

\def\R{{\mathbb{R}}}
\newcommand{\s}{\sigma}
\renewcommand{\a}{\alpha }
\newcommand{\eps}{\varepsilon}

\renewcommand{\t}{\top}
\newcommand{\ip}{(i-p)}
\newcommand{\ipnc}{(i-p-nc)}
\newcommand{\norm}[1]{\|#1 \|}
\newcommand{\abs}[1]{|#1 |}
\newcommand{\as}{\cdot}
\newcommand{\skalar}[1]{\langle #1 \rangle}
\newcommand{\8}{\infty}

\newcommand{\eqdist}{\stackrel{\mathrm{law}}{=}}

\newcommand{\od}{o}   

\newtheorem{theorem}{Theorem}[section]

\newtheorem{corollary}[theorem]{Corollary}

\newtheorem{lemma}[theorem]{Lemma}

\newtheorem{proposition}[theorem]{Proposition}

\theoremstyle{definition}

\newtheorem{remark}[theorem]{Remark}

\begin{document}

\maketitle

\begin{abstract}
	{In recent works on the theory of machine learning, it has been observed that heavy tail properties of Stochastic Gradient Descent (SGD) can be studied in the probabilistic framework of stochastic recursions. In particular, G{\"{u}}rb{\"{u}}zbalaban et al. \cite{Guerbuezbalaban2021} considered a setup corresponding to linear regression for which iterations of SGD can be modelled by a multivariate affine stochastic recursion $X_k=A_k X_{k-1}+B_k$, for independent and identically distributed pairs $(A_k, B_k)$, where $A_k$ is a random symmetric matrix and $B_k$ is a random vector. In this work, we will answer several open questions of the quoted paper and extend their results by applying  the theory of irreducible-proximal (i-p) matrices.}
\end{abstract}



\section{Introduction}

\subsubsection*{Stochastic Gradient Descent}

Consider the following learning problem: For $d, p \ge 1$ let $\mathscr{H}$ and $\mathscr{Z}$ be subsets of $\R^{d}$ and $\R^{p}$, respectively. We consider $\mathscr{H}$ to be a set of parameters and $\mathscr{Z}$ the domain of the observable data. Let $\ell : \mathscr{H} \times \mathscr{Z} \to [0, \infty)$ be a measurable function, called {\em loss function}. The aim is to minimize the {\em risk function} $L_\mathscr{D}(x):= \E \ell(x,Z)$ where $Z$ is a random variable on $\mathscr{Z}$ that is assumed to generate the data, and the law $\mathscr{D}$ of Z is unknown. 

Upon observing a sample $S=(z_1, \dots, z_m)$ of realisations of $Z$, we can define the {\em empirical risk function} by
$$ L_S(x)~:=~\frac{1}{m} \sum_{i=1}^m \ell(x,z_i).$$
If $\mathscr{H}$ is convex and $\ell$ is differentiable on  $\mathscr{H}$ and convex, we can employ the gradient descent algorithm in order to minize $L_S$: Starting with $x_0 \in \mathscr{H}$, we recursively define for $n \in \N$
$$ x_n ~=~ x_{n-1} - \eta \nabla L_s(x_{n-1}) ~=~ x_{n-1} - \eta \Big( \frac1n \sum_{i=1}^n \nabla \ell(x_{n-1}, z_i) \Big)$$
with a parameter $\eta>0$, called the {\em step size}. The parameter $\eta$ has to be chosen sufficiently small, such that the first-order approximation of the function about $x_{n-1}$ by its gradient is good enough.

In {\em stochastic gradient descent} (see e.g. \cite[Ch. 14]{ShalevShwartz2014}), the idea is to use at each step only a random subsample $(z_{i,n})_{1 \le i \le b}$ of size $b$ of the observations, to calculate the gradient:
\begin{equation}\label{eq:iteration1}
	x_n = x_{n-1} -  \frac{\eta}{b} \sum_{i=1}^b \nabla \ell(x_{n-1},z_{i,n}) ~=:~ x_{n-1} - g_n
\end{equation}

By choosing the subsample at random, the gradient term $g_n$ becomes a random variable. If in each step, a new sample is drawn independently of the previous samples, then the successive gradient terms $(g_n)_{n \ge 1}$ are independent as well. However, the $(g_n)_{n \ge 1}$ are not identically distributed, since the gradients are evaluated at different points $x_{n-1}$ in each step.

To analyze the algorithm, it is of prime interest to study the random distance between $x_n$ and local minima. 
 
In \cite{Guerbuezbalaban2021}, the following particular setting of a quadratic loss function has been considered:
The parameter set is $\mathscr{H}=\R^d$, the observed data $z \in\mathscr{Z}=\R^d \times \R$; and with $z=(a,y)$ the loss function is given by
$$ \ell(x,a,y) = \frac12 (a^\top x-y)^2$$
This corresponds to the setting of linear regression where $y \in \R$ would be the target variable and $a$ the vector of explanatory variables; $L_S(x)$ is then equal to the mean squared error.

In this setting, the function $L_S(x)$ has a unique global minimum $x_*$, and we can expand the gradient of $\ell$ around $x_*$:
\begin{equation}
	\nabla \ell (x,a,y) ~=~ a(a^\top x-y) ~=~ -ya + (aa^\top) x_*+ (aa^\top) (x-x_*) 
\end{equation}
Note that $a$ and $x$ are $d$-dimensional column vectors, $a^\top$ is a row vector, $aa^\top$ is the $(d \times d)$-rank-one projection matrix onto $a$. Using this in \eqref{eq:iteration1}, we obtain for the distance $x_n-x_*$ that
\begin{equation}
	x_n -x_* = x_{n-1} - x_* - \frac{\eta}{b} \sum_{i=1}^b \big(a_i a_i^\top\big) (x_{n-1}-x_*) + \frac{\eta}{b} \sum_{i=1}^b \Big( y_i a_i - (a_i a_i^\top) x_* \Big)
\end{equation}	
Let $\Id$ be the $d\times d$ identity matrix. Denoting 
\begin{equation}\label{eq:ABX.SDG}
	X_n:=x_n-x_*, \qquad A_n:= \big(\Id - \tfrac{\eta}{b} \sum_{i=1}^b a_i a_i^\top\big), \qquad B_n := \tfrac{\eta}{b} \sum_{i=1}^b \big( y_i a_i - (a_i a_i^\top) x_* \big),
\end{equation}
we see that $X_n$ satisfies an {\em affine stochastic recursion}
\begin{equation}\label{eq:SRE1}
	X_n ~=~ A_n X_{n-1} + B_n,
\end{equation}
with $(A_n,B_n)$ being now an independent and identically distributed (i.i.d.) sequence of random matrices and vectors.

\subsubsection*{Summary of results}

It is known (\cite{AM2012,Guivarch2016,Kesten1973}, see also the book \cite{Buraczewski2016}) that under mild assumptions, stationary solutions to such recursive equations will exhibit heavy tails even if the law of $a_i$ has all the moments. This means that there is a substantial probability that iterations of the stochastic gradient descent go far away from the minimum, due to the randomness. 

Here, a random variable $X$ is a stationary solution to \eqref{eq:SRE1} if $X$ has the same law as $A_1X+B_1$, and $X$ is independent of  $(A_1,B_1)$. The main condition for the existence (and uniqueness) of such a stationary solution is that the {\em top Lyapunov exponent} is negative, see below for details. We say that (the law of) $X$ has {\em Pareto tails with tail index $\alpha$}, if there is $c>0$ such that $\lim_{t\to \8}{t^\alpha} \P(|X|>t)=c$.  This is a particular instance of heavy-tail behaviour.  Intriguing questions are: How is the tail behaviour affected by the subsample size $b$, the step size $\eta $ or the distribution generating $(a_i,y)$? What can be said about the law or the moments of finite iterations $X_n$?

In \cite{Guerbuezbalaban2021}, these questions have been studied under a density assumption on the law of $a_i$ and relevant results have been proved when $a_i\sim \mathcal{N}(0,\s ^2I_d)$ is standard normal.  We are going to answer several open questions of that paper and to extend their results beyond the density assumption by applying  the theory of i-p (irreducible-proximal) matrices. In particular, this allows to cover the setup where $a_i$ are samples drawn from a finite set; {\em i.e.}, the case when the law of $a_i$ is supported on a finite set. The tail behaviour of $X$ is described in Theorem \ref{th:GuiLeP}, which we quote from \cite{Guivarch2016}. However, checking its assumptions is not straightforward\footnote{Indeed, there are gaps in \cite{Guerbuezbalaban2021}: In Theorem 2, $M$ does not have a Lebesgue density on $\R^{d^2}$ that is positive in a neighbourhood of the identity matrix $\Id$ since the law of $M$ is concentrated on symmetric matrices, which constitute a $d(d+1)/2$-dimensional subspace, of $d^2$-Lebesgue measure 0. In \cite[Lemma 20]{Guerbuezbalaban2021supp}, the inequality D.18, $\log(\frac{1}{1-x})\le 2x$, is wrong for $x$ close to 1.} since the result requires finiteness of $\E \det(A)^{-\epsilon}$ for some $\epsilon>0$, a property that we will prove in detail in Section \ref{sect:momentbounds} under a density assumption in the law of $A$.

Moreover, we want to consider more general settings where
$A_n = \Id - \tfrac{\eta}{b} \sum_{i=1}^b H_{i,n}$ for i.i.d.\ random symmetric $d \times d$-matrices $H_{i,n}$. This is motivated by the observation (cf. \cite{Guerbuezbalaban2021,Hodgkinson2021}) that for a twice differentiable loss function $\ell$, we may perform a Taylor expansion of the gradient about $x_*$, to obtain
\begin{align}
	x_n - x_* ~\approx&~
	\big( \Id - \frac{\eta}{b}\sum_{i=1}^b H_\ell(x_*,z_{i,n})\big) (x_{n-1} - x_*) - \frac{\eta}{b}\sum_{i=1}^b \nabla \ell(x_*,z_{i,n}) \label{eq:SDGgeneral},
\end{align}
where $H_\ell$ denotes the Hessian of $\ell$. However, one has to keep in mind that such an approximation is in general valid only for $x_n$ close to $x_*$. 

In this general setting we are able to prove that $\E |X_n|^{\a}$ grows linearly with $n$, see Theorem \ref{th:growmoments}, which  considerably improves the result of \cite[Prop. 10]{Guerbuezbalaban2021}.
To obtain a tractable expression for the value of $\alpha$ as well as for the top Lyapunov exponent, we need an additional condition: that the law of $A$ is invariant under rotations, which in particular includes the Gaussian case studied in \cite{Guerbuezbalaban2021}.
We derive a simple formula for the Lyapunov exponent and an appropriate moment generating function, see Lemma \ref{lem:uniform measure}, which  allows us to describe completely the behaviour of $\a $ as a function of $\eta $ and $b$, see Theorem \ref{th:behavioralpha}.

We will introduce all relevant notation, description of the models and assumptions in Section \ref{sect:preliminaries}, in order to state our main results in Section \ref{sect:results}; the proofs of which are then given in the subsequent sections.

\section{Assumptions, Notations and Preliminaries}\label{sect:preliminaries}

We assume that all random variables are defined on a generic probability space $(\Omega, \mathcal F, \P)$.
Let $d \ge 1$. Equip $\R^d$ with the euclidean norm $\abs{\cdot}$ and let $\norm{\cdot}$ be the corresponding operator norm for matrices. In addition, for a $d \times d$-matrix $g$, 
$$ \norm{g}_F:= \Big( \sum_{1\le i,j\le d} g_{ij}^2\Big)^{1/2}$$
denotes the Frobenius norm. We write $\Sym$ for the set of symmetric $d \times d$-matrices,  $\Mset{d}$ for the group of invertible $d \times d$-matrices over $\R$ and $O(d)$ for its subgroup of orthogonal matrices. The identity matrix is denoted $\Id$.  Write $S:=S^{d-1}$ for the unit sphere in $\R^d$ with respect to $\abs{\cdot}$.
For $x \in S$ and $g \in \Mset{d}$, define the action of $g$ on $x$ by 
$$g \cdot x:= \frac{gx}{\abs{gx}} \in S.$$
We fix an orthonormal basis $e_1, \dots, e_d$ of $\R^d$. We write $\Cf(S)$ for the set of continuous functions on $S$. The uniform measure $\sigma $ on $S$ is defined by
\begin{equation}\label{eq:uniform.measure}
	\int _Sf(y)\ d\sigma (y):= \int _{O(d)} f(\od\cdot e_1)\ d\od, \qquad \text{for all }f \in \Cf(S),
\end{equation}
where $d\od $ is the Haar measure on $O(d)$ normalized to have mass 1.
For $g \in \Mset{d}$ we introduce the quantity \begin{equation} \label{eq:Ng}
	N(g):=\max\{\norm{g},\norm{g^{-1}}\}.
\end{equation}
For an intervall $I \subset \R$, the set of continuously differentiable functions on $I$ is denoted $\Cf^{1}(I)$. For a random variable $X$, we denote by $\supp(X)$ its support, i.e., the smallest closed subset $E$ of its range with $\P(X \in E)=1$. We also use this notation for measures, then the support is the smallest closed set of full measure.

\subsubsection*{Stochastic Recurrence Equations}
Given a sequence of i.i.d. copies $(A_k,B_k)_{k \ge 1}$ of a random pair $(A,B) \in \Mset{d} \times \Rd$, which are also independent of the random vector $X_0 \in \R^d$, we consider the sequence of random $d$-vectors $(X_k)_{k \ge 1}$ defined by the affine stochastic recurrence equation
\begin{equation}\label{eq:SRE}
	X_k ~=~ A_k X_{k-1} + B_k, \qquad k \ge 1.
\end{equation}
The study of the equation \eqref{eq:SRE} goes back to \cite{Kesten1973}, see the book \cite{Buraczewski2016} for a comprehensive account. Let $\Pi_n := \prod_{i=1}^n A_1 \cdots A_n$. Assume $\E \log^+ \norm{A} < \infty$, $\E \log^+ \abs{B}<\infty$ and that the {\em top Lyapunov exponent}
\begin{equation}\label{eq:lyapunov exp}
	\gamma~:=~\lim_{n \to \infty} \frac{1}{n} \log \norm{\Pi_n}  \quad \text{a.s.}
\end{equation}
is negative. Then there exist a unique stationary distribution for the Markov chain defined by \eqref{eq:SRE} and its law is given by the then almost surely convergent series
\begin{equation}\label{eq:R}
	R:= \lim_{n \to \infty}	R_n ~:=~ \lim_{n \to \infty} \sum_{k=1}^n \Pi_{k-1} B_k \quad \text{a.s.}
\end{equation}
see \cite[Theorem 4.1.4]{Buraczewski2016}. Indeed, let us write
$$ R^{(1)}~:=~ \lim_{n \to \infty} R_n^{(1)} ~:=~ \lim_{n \to \infty} \sum_{k=1}^n A_2 \cdots A_{k} B_{k+1},$$
{\em i.e.}, we shift all indices by 1. Then $R_n \eqdist R_n^{(1)}$, $R \eqdist R^{(1)}$ and 
$$ R_{n+1}=A_1 R_n^{(1)} +B_1.$$
Hence (upon taking limits $n\to \infty$), we see that $R$ satisfies the distributional equation $R \eqdist A R + B$, where $A,B$ are independent of $R$. Let $$ I_k:= \{ s \ge 0 \, : \, \E \norm{A}^s < \infty \} \quad \text{and} \quad s_0:=\sup I_k.$$ For any $s \in I_k$  define the quantity
\begin{equation}\label{def:ks}
	k(s) ~:=~ \lim_{n \to \infty} \left(\E{\norm{A_n \ldots A_1 }^s} \right)^\frac{1}{n}. 
\end{equation}
Considering Eq. \eqref{eq:R}, a simple calculation gives that $k(s)<1$ together with $\E \abs{B}^s<\infty$ implies $\E \abs{R}^s < \infty$ (see also \cite[Remark 4.4.3]{Buraczewski2016}).

\subsubsection*{The Setup: $A = \Id - \xi H$ for $H \in \Sym$}

Motivated by the application to stochastic gradient descent, we want to study matrices $A$ of the form $A=\Id-\xi H$, where $\xi >0$ and $H$ is a random symmetric matrix. More specifically, we want to consider the following two models, both for $b \in \N$ and $\eta>0$. Firstly, a setup where $B$ is arbitrary and $A$ is composed from random symmetric matrices. Given  a tupel $(H_i)_{1\le i \le b}$ of i.i.d.\ random symmetric $d \times d$-matrices, let
\begin{equation}\label{Symm}\tag{Symm}
	A ~=~ \Id - \frac{\eta}{b} \sum_{i=1}^b H_i, \qquad B \text{ a random vector in } \R^d. 
\end{equation}
Secondly, we consider a setup where $A$ is composed from rank-one projections and $B$ is a weighted sum of the corresponding eigenvectors.
\begin{equation}\label{Rank1}\tag{Rank1}
	A =\Id - \frac{\eta}{b} \sum_{i =1}^b a_i a_i^\t, \qquad B ~=~ \frac{\eta}{b} \sum_{i=1}^b a_i y_i
\end{equation}
where $(a_i,y_i)$ are  i.i.d. pairs in $\R^d \times \R$.
Of course, model \eqref{Symm} contains model \eqref{Rank1} as a particular case. Therefore, we are going to study a general model wherever possible, and only restrict to the particular case (that still covers \cite{Guerbuezbalaban2021} completely) when this is necessary to obtain more precise results.

\subsubsection*{Geometric Assumptions}
Let $\mu_A$ be the law of $A$ and let $G_A$ denote the closed semigroup generated by the $\supp(\mu_A)$.  It will be required that $G_A \subset GL(d,\R)$, {\em i.e.}, the matrix $A$ is invertible with probability 1. This is satisfied {\em e.g} when $\mu_A$ has a density w.r.t.\ Lebesgue measure on the set of symmetric matrices: the set of matrices with determinant 0 is of lower dimension.

Here and below, a matrix is called {\em proximal}, if it has a unique largest eigenvalue (with respect to the absolute value) and the corresponding eigenspace is one-dimensional.
We say that $\mu_A$ satisfies an irreducible-proximal condition (short: condition \ip), if the following holds.
\begin{enumerate}
	\item {\em Irreducibility condition}: There exists no finite union $\mathcal{W}=\bigcup_{i=1}^n W_i$ of proper subspaces $W_i \subsetneq \Rd$ which is $G_A$-invariant.
	\item {\em Proximality condition}: $G_A$ contains a proximal matrix.
\end{enumerate}
The irreducibility condition does not exclude a setting where $G_A$ leaves invariant a cone (for example, if all matrices are nonnegative in addition), see \cite[Proposition 2.14]{Guivarch2016}. Since this setting is not relevant to the applications that we have in mind, we will exclude it and consider the following assumption in all of the subsequent results:
\begin{equation}\label{eq:ipnc}
	G_A \subset GL(d,\R),\ \mu_A \text{ satisfies \ip\ and there is no $G_A$-invariant cone} \tag{i-p-nc}
\end{equation}

For some results, we will require more specific assumptions which we describe now.
We say that $\mu _A$ is {\em invariant under rotations} if for every $\od\in O(d)$, $\od H\od^T$ has the same distribution law as $H$. We will make use of the following condition.
\begin{equation}\tag{rotinv-p}\label{rotinv}
	G_A \subset GL(d,\R),\ \mu_A \text{ is invariant under rotations and } G_A \text{  contains a proximal matrix }
\end{equation}
For model \eqref{Rank1}, we can reformulate condition \eqref{rotinv} in terms of $a_i$ by the following observation.
If a random vector $a\in \R ^d$ has the property that for every $\od \in O(d)$, $ka$ has the same law as $a$, then  then the law of $H:=aa^T$ is invariant under rotations.

We will sometimes assume a Gaussian distribution for $a_i$'s.
\begin{equation}\tag{Rank1Gauss}\label{Rank1Gauss}
	A, B \text{ are of the form } \eqref{Rank1},  a_i \sim \mathcal{N}(0,\Id) \text{ and independent of } y_i \sim \mathcal{N}(0,1)
\end{equation}

Note that \eqref{Rank1Gauss} is assumed throughout in \cite{Guerbuezbalaban2021}; we will show that indeed all their results hold at least under \eqref{rotinv}, some already under \eqref{eq:ipnc}. 

\subsubsection*{Assumptions on Moments and Decay}

Certain results will require finiteness of $\E \log N(A)$ or $\E N(A)^\epsilon$ for some $\epsilon>0$, which means in particular that we require finiteness of small (or logarithmic) moments of $\norm{A^{-1}}$ (see \eqref{eq:Ng}). Such a property is difficult to check in general. This is why we provide two sufficient conditions, formulated in terms of densities.
The first one is for the model \eqref{Symm}, where $A=I-\xi H:=I-\frac{\eta}{b}\sum_{i=1}^b H_i$.
\begin{align}\notag
	&\text{The law of $H$ has a density $f$ w.r.t. the Lebesgue measure on  $\Sym$ that satisfies} \\  & \qquad f(g)\leq  C\left ( 1+\norm{g}_F^2\right )^{-D} \qquad \text{
		for some $D>\frac{d(d+1)}{4}$} \label{decay}\tag{Decay:Symm}
\end{align}
The second one is for the model \eqref{Rank1}.
\begin{align}\label{decay1}\tag{Decay:Rank1}
	\text{The law of $a_1,..,a_b$ on $\R ^{db}$ has a joint density $f$ that satisfies} \\
	f(a_1,...,a_b)\leq C\left ( 1+\sum _i^b| a_i| ^2\right )^{-D} \qquad \text{	for some $D>\frac{db}{2}$.}  \notag
\end{align}
Under \eqref{decay} or \eqref{decay1}, small moments of $\| A^{-1}\| $ are finite which simplifies the assumptions, which is desirable from the point of view of applications.

\subsubsection*{Preliminaries}

For $s \in I_k$, we define the operators $\Ps$ and $\Pst$  in $\Cf(S)$ as follows: For any $f \in \Cf(S)$,
\begin{equation}\label{eq:Ps}  \Ps f(x) := \E \abs{A x}^s \, f(A \as x), \quad  \qquad  \Pst f(x) := \E \abs{A^\t x}^s \, f(A^\top \as x) . \end{equation}
Properties of both operators will be important in our results. 
\begin{proposition}\label{prop:transferoperators}
	Assume that $\mu_A$ satisfies \ipnc\ and let $s \in I_k$. Then the following holds.
	The spectral radii $\rho(\Ps)$ and $\rho(\Pst)$ both equal $k(s)$ and there is a  unique probability measure $\nus$ on $\S$ and a unique function $\es \in \Cf(\S)$ satisfying 
	$$ \int \es(x)\, \nus(dx)=1, \quad \Ps \es = k(s) \es \quad \text{ and } \quad \Ps \nus = k(s) \nus.$$
	Further, the function $\es$ is strictly positive.
	Also, there is a unique probability measure $\nust$  satisfying $\Pst \nust = k(s) \nust$ and both the $\supp(\nus)$ and $\supp(\nust)$ are not concentrated on any hyperplane.  There  is $c>0$ such that 
	$$ \es(x) ~=~ c \int_S \abs{\skalar{x,y}}^s \nust(dy).$$
	The function $s \mapsto k(s)$ is log-convex on $I_k$, hence continuous on $\interior{I_k}$ with left- and right derivatives and there is a constant $C_s$ such that for all $m \in \N$,
	$$\E \norm{\Pi_m}^s \le C_sk(s)^m.$$ 
\end{proposition}

\begin{proof}[Source]
	This is a combination of \cite[Theorem 2.6]{Guivarch2016} and \cite[Theorem 2.16]{Guivarch2016}; where \ipnc\ corresponds to Case I. The results concerning the supports as well as the bound for $k(s)$ are proved in \cite[Lemma 2.8]{Guivarch2016}.
\end{proof}

\begin{proposition}\label{prop:diffk}
	Assume \ipnc\ and that 
	\begin{equation}
		\E (1+\| A\| ^s)\log N(A)<\8 
	\end{equation} 
	for some $s>0$.
	Then $k$ is continuously differentiable on $[0,s]$ with $k'(0)=\gamma$. 
\end{proposition}

\begin{proof}[Source]
	\cite[Theorem 3.10]{Guivarch2016}
\end{proof}

\begin{remark}
	Hence, a sufficient condition for $\gamma <0$ is $\E \norm{A_1}<1$ (together with $\E (1+\norm{A}) \log N(A)<\infty$): Observing that $k(0)=1$ and $k(1)\le \E \norm{A_1}<1$ by submultiplicativity of the norm, the convexity of $k$ (see Prop. \ref{prop:transferoperators}) implies that  $k'(0)=\gamma<0$.
\end{remark}

\section{Main results}\label{sect:results}

In this section, we state our main results, the proofs of which are deferred to the subsequent sections. We formulate all results with the minimal set of assumptions that is required; note that these assumptions are in particular satisfied for model \eqref{Rank1Gauss}.

\subsection{Heavy tail properties}

For completeness, we start by quoting the fundamental result about the tail behavior of $R$ from \cite{Guivarch2016} and adopt it to our notation. 

\begin{theorem}\label{th:GuiLeP}
	Assume \eqref{eq:ipnc}, $\gamma <0$ and $P(Ax+B=x)<1$ for all $x \in \R^d$. Assume further that there is $\alpha \in \mathrm{interior}(I_k)$ with $k(\a)=1$, $\E |B|^{\alpha + \epsilon} < \infty$ and $\E \norm{A}^{\a } N(A)^\epsilon < \infty$ for some $\epsilon>0$. Then there is $C>0$ such that
	$$ \lim_{t \to \infty} t^\alpha \P \left( |R|>t, \frac{R}{|R|} \in \cdot \right) ~=~ C\, \nu_\alpha(\cdot).$$
	
\end{theorem}

\begin{proof}[Source]
	This is \cite[Theorem 5.2]{Guivarch2016}.
\end{proof}

\begin{remark}
	This fundamental results gives {\em i.a.} that $R$ is in the domain of attraction of a multivariate $\alpha$-stable law, if $\alpha \in (0,2)$. In fact, even more can be said about sums of the iterations. Write $S_n:=\sum_{k=1}^n R_k$ and let $m:=\E R$ if $\alpha>1$ (then this expectation is finite.) It is shown in \cite[Theorem 1.1]{Gao2015}	that under the assumptions of Theorem \ref{th:GuiLeP}, the following holds.  If $\alpha>2$, then $\frac{1}{\sqrt{n}} (S_n-nm)$ converges in law to a multivariate Normal distribution.  If $\alpha=2$, then			$\frac{1}{\sqrt{n \log n}} (S_n-nm)$ converges in law to a multivariate Normal distribution.  If $\alpha \in (0,1) \cup (1,2)$, let $t_n:=n^{-1/\alpha}$ and $d_n= n t_n m \mathds{1}_{\{\alpha>1\}}$. 
	Then $(t_n S_n - d_n)$ converges in law to a multivariate $\alpha$-stable distribution.
	In all cases, the limit laws are fully nondegenerate.
	
\end{remark}

Theorem \ref{th:GuiLeP} gives in particular that $\E |R|^\alpha=\infty$.
Our first result shows that $\E |R_n|^\alpha$ is of order $n$ precisely.

\begin{theorem} \label{th:growmoments} Assume \eqref{eq:ipnc}, $\gamma <0$, $P(Ax+B=x)<1$ for all $x \in \R^d$  and that there exists $\alpha \in I_k$ with $k(\alpha)=1$ and $\E |B|^\alpha<\infty$. Then 
	\begin{equation}\label{eq:momentgrowth}
		\limsup \frac{1}{n}\E |R_{n}|^{\a}<\8, \quad  \liminf \frac{1}{n}\E |R_{n}|^{\a}>0.
	\end{equation}	
	
\end{theorem}
\begin{remark} Our result improves \cite[Prop. 10]{Guerbuezbalaban2021}, where it was shown that $\E |R_n|$ is at most of order $n^\alpha$. 
	Analogous results were obtained in \cite{BurDamZen2016} for the stochastic recursion \eqref{eq:SRE1} in the one dimensional case as well as for a multidimensional case  with $A_n$ being similarities. 
\end{remark}

While Theorem \ref{th:GuiLeP} is concerned with the tail behaviour of the stationary distribution $R$, our next theorem gives an upper bound on the tails of finite iterations.

\begin{theorem}
	\label{thm:tailsRn}
	Assume \eqref{eq:ipnc}, $\gamma <0$ and that there is $\alpha \in I_k$ with $k(\a)=1$. Assume further that there is  $\epsilon>0$ such that $k(\alpha+\epsilon)<\infty$ and $\E |B|^{\alpha+\epsilon}<\infty$.
	Then for each $n \in N$, there is a constant $C_n$, such that for all $t>0$
	$$ \P ( \abs{R_n} > t) ~\le~ C_n t^{-(\alpha + \epsilon)} $$
\end{theorem}

Note that the constant $C_n$ changes with $n$, but is independent of $t$. This shows that for any fixed $n$, the tails of $R_n$ are of (substantially) smaller order than the tails of the limit $R$. 

\medskip

Theorem \ref{th:GuiLeP} states that the directional behaviour of large values of $R$ is governed by $\nu_\alpha$, which is given by Proposition \ref{prop:transferoperators} in an abstract way: as the invariant measure of the operator $P^\alpha$. The next theorem allows to obtain a simple expression both for $\nu_\alpha$ as well as for the top Lyapunov exponent, under assumption \eqref{rotinv}.

\begin{theorem}\label{lem:uniform measure}
	Assume \eqref{rotinv}. Then for all $s \in I_k$,  $$k(s)= \E | (I-\xi H)e_1| ^s$$ and
	$\nus$ is the uniform measure on $S$. 
	\\	
	If in addition $\E (1+\| A\| ^s)\log N(A)<\8$ holds for some $s>0$, then 
	%
	the Lapunov exponent $\gamma$ is
	\begin{equation*}
		\gamma=k'(0)=\E \log | (I-\xi H)e_1| .
	\end{equation*}
\end{theorem}
\begin{remark}
	Our result strengthens \cite[Theorem 3]{Guerbuezbalaban2021}, where the formulae for $k(s)$ and $\gamma$ were proved for the particular model \eqref{Rank1Gauss}. 
\end{remark}

\subsection{How does the tail index depend on $\xi$?}

For $A=I-\xi H$ and the corresponding $R$ we intend to determine how the tail index of $R$ depends on $\xi $, assuming model \eqref{rotinv}. 
For this model, Theorem \ref{lem:uniform measure} provides us with the formulae
\begin{equation}\label{eq:kxis}
	k(s)=\E | (I-\xi H)e_1| ^s=:h(\xi,s), \qquad \gamma =\E \log  | (I-\xi H)e_1|.
\end{equation}
We introduce the function $h$ here to highlight the dependence on both $\xi,s$ and to remind the reader that it is equal to the spectral radius $k(s)$ only for this particular model.
By Theorem \ref{th:GuiLeP}, the tail index $\a (\xi ) $ of $R$ corresponding to $\xi $ satisfies 
\begin{equation}\label{eq:tailind}
	h(\xi, \a (\xi))=1	
\end{equation}

For fixed $\xi>0$, $h$ is convex as a function of $s$ (see Prop. \ref{prop:transferoperators}), with $h(0)=1$. Thus $\gamma=\frac{\partial{d}h}{\partial{ds}}'(\xi,0)<0$ is a necessary condition for the existence of $\alpha>0$ with $h(\xi,\alpha)=1$. 

A simple calculation shows that $\gamma \to \8$ when $\xi \to \8$ and so \eqref{eq:tailind} may happen only for $\xi $ from a bounded set. However, for an arbitrary law of $H$ it is difficult to determine the set $U:=\{ \xi: \E \log  | (I-\xi H)e_1|<0 \}$ directly. Therefore, we will provide sufficient conditions for $U \neq \emptyset$.

\begin{theorem}\label{th:shapes} \label{th:behavioralpha}  Consider $A=\Id - \xi H$ with variable $\xi>0$.
	Assume \eqref{rotinv} and that 
	$\supp(H)$ is unbounded. Assume $s_0=\infty$ and that	
	\begin{equation}\label{eq:Hpos2}
		\E  \langle He_1,e_1\rangle >0 \quad \text{ and } \quad \E |\langle He_2,e_1\rangle |^{-\delta} <\infty \text{ for every } \delta \in (0,1).
	\end{equation}
	Then the following holds.
	\begin{enumerate}
		\item\label{k2} 	There is a unique $\xi _1>0$ such that $h(\xi _1,1)=1$  and for every $\xi \in (0,\xi_1)$ there is a unique $\a =\a (\xi )>1$ such that  $h(\xi, \a (\xi))=1.$	In particular,
		$\E \log  | (I-\xi H)e_1|<0$ for all $\xi \in (0,\xi _1]$ and the set $U$ is nonempty.
		\item\label{k3} The function $\xi \mapsto \a (\xi)$ is strictly decreasing and in $\Cf^{1}\big( (0,\xi_1) \big)$,  with
		\begin{equation}\label{eq:limits}
			\lim _{\xi \to 0^+}\a (\xi )=\8,\quad \quad \lim _{\xi \to \xi ^-_1}\a (\xi )=1.
		\end{equation}
		\item\label{k4} For $\xi \in U$ with $\xi >\xi _1$ it holds $\alpha(\xi )<1$.
	\end{enumerate}
\end{theorem}

	\begin{remark}
		Our result generalizes \cite[Theorem 3]{Guerbuezbalaban2021} (which covers \eqref{Rank1Gauss} only) and clarifies  the proof of the differentiability both of $k$ and $\alpha$. The second part of condition \eqref{eq:Hpos2} together with $s_0=\infty$ implies $\E (1+ \norm{A}^s) \log N(A)<\infty$ for all $s>0$, see Lemma \ref{lem:partialdiffks} for details.
	\end{remark}
	
	\begin{remark} The assumption $\E \langle He_1,e_1\rangle >0$ will imply that $U \neq \emptyset$, while the second assumption in \eqref{eq:Hpos2} is used to assure that $h\in C^1\left ( (0,\8 )\times (1,\8 )\right)$, which we need to employ the implicit function theorem to study properties of $\xi \mapsto \alpha(\xi)$. Note that  $\E \langle He_1,e_1\rangle >0$ is automatically satisfied when $H$ is positive definite e.g.
		in the \eqref{Rank1} case. 
		
		For $\xi \in U$ with $\xi > \xi_1$, there still exists a unique $\alpha(\xi)$ which then satisfies $\alpha(\xi)<1$. However, we are not able to prove the $\Cf^1$ regularity of $h(\xi,s)$ when $s<1$ since $\partial h/\partial s$ is singular at $\xi=0$, see the proof of Lemma \ref{lem:diff} for details.

	\end{remark}
	
	\begin{figure}
		\centering
		\includegraphics[width=0.9\linewidth]{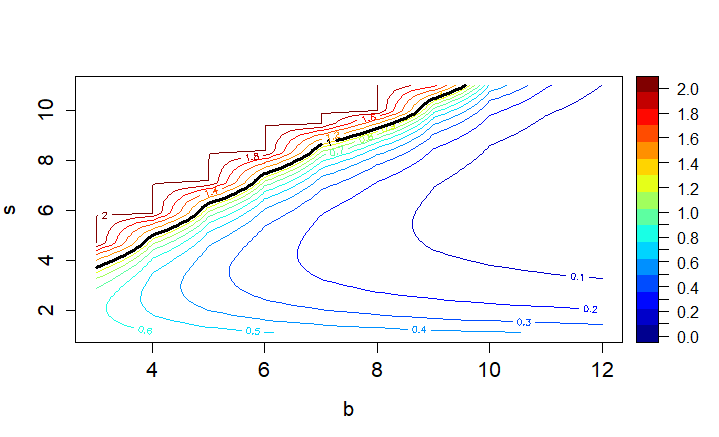}
		\caption{Contour plot of $h$ as a function of $b$ and $s$, for model \eqref{Rank1Gauss} with $d=2$ and $\eta=0.75$. The black line is the contour of $k \equiv1$. The values of $h$ have been cutted at level 2 for a better visualization.}
	\end{figure}
	
	\begin{figure}
		\centering
		\includegraphics[width=0.9\linewidth]{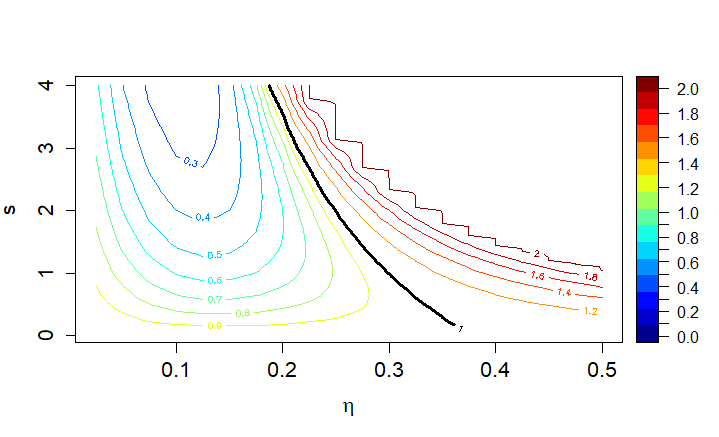}
		\caption{Contour plot of $h$ as a function of $\eta$ and $s$, for model \eqref{Rank1Gauss} with $d=2$ and $b=5$. The black line is the contour of $k\equiv1$. The values of $h$ have been cutted at level 2 for a better visualization. }
	\end{figure}

	\subsection{Checking the assumptions - the model \eqref{Rank1Gauss}}

	The following results confirm the hierarchy of our assumptions. 

	\begin{lemma}\label{lem:hierarchyAssump}
		We have the following implications:
		$$ \text{\eqref{Rank1Gauss}} \ \Rightarrow \ \text{\eqref{rotinv}} \ \Rightarrow \ \text{\eqref{eq:ipnc}} $$
	\end{lemma}
	
	Our next result is concerned with the density of the law of $H=\sum_{i=1}^b a_ia_i^\top$ in the \eqref{Rank1Gauss} model. 

	The entries of $g\in \Sym$ will be denoted $g_{ij}$ i.e. $g=(g_{ij})$, where $g_{ij}=g_{ji}$. We identify $\Sym$ with $\R^{d(d+1)/2}$. Observe that then the Frobenius norm of $g$ is bounded by the Euclidean norm of the corresponding vector $x$ of $\R^{d(d+1)/2}$, which contains the diagonal and upper diagonal entries of $g$: $\norm{g}_F \le 2 \abs{x}$.

	\begin{lemma}\label{lem:densityGauss}
		Assume \eqref{Rank1Gauss}. If $b > d+1$, then $H=\sum_{i=1}^b a_ia_i^\top$ has a density $f$ with respect to Lebesgue measure on $\R^{d(d+1)/2}$ that satisfies
		$$ f(x)\leq C\left ( 1+\abs{x}^2\right )^{-D}$$
		for all $D>0$ and some constant $C=C(D)$ depending only on $D$.
	\end{lemma}
	
	In particular it is possible to choose $D>\frac{d(d+1)}{4}$ in the lemma above, so that the condition \eqref{decay} is satisfied.
	The next result gives that the decay assumptions are sufficient for existence of negative moments, which are required in the previous theorems.

	\begin{theorem}\label{th:integrab}\label{cor:N(A)}

		Assume \eqref{Symm} with \eqref{decay} or \eqref{Rank1} with \eqref{decay1}.
		Then for every $0\leq \delta <1\slash 2$ and $\xi $
		\begin{equation}\label{eq:det}
			\E |\det A|^{-\delta } <\8,
		\end{equation}
		and for every $s<s_0$,
		\begin{equation}\label{eq:NAproof}
			\E (1+\norm{A}^s) N(A)^\varepsilon<\infty
		\end{equation} for some $\varepsilon>0$.
		Further, for every $0<\delta <1$
		$$\E |\langle H e_2, e_1\rangle |^{-\delta} <\8.$$
	\end{theorem}
	
	For ease of reference, we end this section by stating that all results are in particular valid for the \eqref{Rank1Gauss} model.

	\begin{corollary}\label{cor:Rank1Gauss}
		Assume \eqref{Rank1Gauss} with $b>d+1$, i.e., 
		$$A =\Id - \frac{\eta}{b} \sum_{i =1}^b a_i a_i^\t, \qquad B ~=~ \frac{\eta}{b} \sum_{i=1}^b a_i y_i$$
		where $a_i$, $1 \le i \le b$ are i.i.d. standard Gaussian random vectors in $\R^d$, independent of $y_i$. Then the assertions of Theorems \ref{th:GuiLeP}, \ref{th:growmoments}, \ref{thm:tailsRn}, \ref{lem:uniform measure} and \ref{th:behavioralpha} hold. The tail index $\alpha$ is strictly decreasing in the step size $\eta$ and strictly increasing in the batch size $b$, provided that $\alpha>1$.
	\end{corollary}

	\section{Finite Iterations: Proofs of Theorems \ref{th:growmoments} and \ref{thm:tailsRn}}\label{sect:tails}
	
	We start by proving \eqref{eq:momentgrowth}, that is, $\E \abs{R_n}$ is of precise order $n$ as $n \to \infty$.
	
	\begin{proof}[Proof of Theorem \ref{th:growmoments}] By Prop. \ref{prop:transferoperators}, for any $x\in \R ^d$, 
		$ r_{\a}(x\slash |x|) |x|^\a  ~=~ c \int_S \abs{\skalar{x,y}}^{\a} \nu ^*_{\a}(dy).$
		Substituting $x=R _n$ and taking expectations of both sides, we have
		\begin{equation*}
			\E r_{\a}\left (\frac{R_n}{|R_n|}\right ) |R_n|^\a  ~=~ c \int_S \E \abs{\skalar{y,R_n}}^{\a} \nu ^*_{\a}(dy)=:b_n
		\end{equation*}
		\noindent \textbf{Step 1}: We shall prove that $\lim_{n \to \infty}\frac{b_n}{n}=C$,
		where
		$$ C =\int _S\E \left [\abs{\skalar{y,A_1R^{(1)} +B_1}}^{\a} -\abs{\skalar{y,A_1R^{(1)} }}^{\a}\right ]\nu ^*_{\a}(dy).$$
		Since $R_n^{(1)}$ has the same law as $R _n$ we may write
		\begin{equation*}
			b_n=\E r_{\a}\left (\frac{R_n^{(1)}}{|R_n^{(1)}|}\right ) |R_n^{(1)} |^\a .
		\end{equation*}
		Using that $r_{\a}$ is an eigenfunction of $P^{\a}$ with the eigenvalue 1, we deduce that
		\begin{align*}
			b_n=&\E \left [P^{\a} (r_{\a})\left (\frac{R_n^{(1)}}{|R_n^{(1)}|}\right ) |R_n^{(1)}|^\a \right ]
			=\E \left [ r_{\a}\left (A_1\cdot \frac{R_n^{(1)}}{|R_n^{(1)}|}\right )
			\left |A_1\frac{R_n^{(1)}}{|R_n^{(1)}|}\right |^{\a} |R_n^{(1)}|^\a \right ]\\
			=&\E \left [ r_{\a}\left (A_1\cdot \frac{R_n^{(1)}}{|R_n^{(1)}|}\right )
			\left |A_1R_n^{(1)} \right |^\a \right ]
			=c \int_S \E \abs{\skalar{y,A_1R_n^{(1)}}}^{\a} \nu ^*_{\a}(dy).
		\end{align*}
		We used Fubini in the last line to change integral and expectation. Using that $R_{n+1}=A_1R_n^{(1)}+B_1$ a.s, we obtain for $b_{n+1}$ the expression
		\begin{equation*}
			b_{n+1}=c \int_S \E \abs{\skalar{y,A_1R_n^{(1)} +B}}^{\a} \nu ^*_{\a}(dy)
		\end{equation*}
		and consequently 
		\begin{equation}\label{eq:bnbn1}
			b_{n+1}-b_n=\int _S\E \left [\abs{\skalar{y,A_1R_n^{(1)}  +B_1}}^{\a} -\abs{\skalar{y,A_1R_n^{(1)}}}^{\a}\right ]\nu ^*_{\a}(dy).
		\end{equation}
		If we are allowed to take the limit inside the expectation, then it immediately follows $\lim _{n\to \8 } b_{n+1}-b_n=C$; and, using a telescoping sum, also the desired result  $\lim _{n\to \8} \frac{b_n}{n}=C.$
		
		Hence, we have to come up with an uniform bound for the term inside the expectation.
		
		\noindent	 \textbf{Case $\a \leq 1$}:
		\begin{equation*}
			\Big|\abs{\skalar{y,A_1R_n^{(1)} +B_1}}^{\a} -\abs{\skalar{y,A_1R_n^{(1)} }}^{\a}\Big|\leq \Big|\skalar{y,A_1R_n^{(1)} +B_1} -\skalar{y,A_1R_n^{(1)} }\Big|^{\a}\leq |\skalar {y, B_1}|^{\a }
		\end{equation*}
		and $\int _S\E |\skalar {y, B_1}|^{\a } \nu_\alpha^*(dx) \le \E |B_1|^\alpha < \8$. 
		
		\noindent	 	 \textbf{Case $\a > 1$}: We use the inequality $||a|^{\a}-|b|^{\a}|\leq \a |a-b|\max \{|a|^{\a -1}, |b|^{b-1} \}$ to infer
		\begin{align*}
			&\Big|\abs{\skalar{y,A_1R_n^{(1)} +B_1}}^{\a} -\abs{\skalar{y,A_1R_n^{(1)} }}^{\a}\Big| \\[.2cm]
			\leq&~ \a  |\skalar {y, B_1}|\max \left \{\abs{\skalar{y,A_1R_n^{(1)} +B_1}}^{\a -1}, \abs{\skalar{y,A_1R_n^{(1)} }}^{\a -1} \right \}\\[.2cm]
			\leq&~ \a 2^\a | B_1| \left ( |A_1R_n^{(1)} |^{\a -1} +|B_1|^{\a -1}\right ) 
			~\leq~ \a 2^\a\left (|B_1|^{\a}+|B_1|\| A_1\| ^{\a -1}|R_n^{(1)} |^{\a -1}\right ). 
		\end{align*}
		Observe that $|B_1|\| A_1\| ^{\a -1}$ and $|R_n^{(1)} |^{\a -1}$ are independent and 
		\begin{equation*}
			\E |B_1|\| A_1\| ^{\a -1}\leq (\E |B_1|^{\a})^{1\slash \a}(\E \|A_1\|^{\a})^{(\a -1)\slash \a}<\8 .
		\end{equation*}
		So it remains to dominate $ |R_n^{(1)} |^{\a -1}$ independently of $n$ by an integrable function. We have
		$	|R_n|\leq \sum _{i=1}^{\8 }\| \Pi_{i-1}\| |B_i|.$
		Hence for $\a -1 \leq 1$, 
		\begin{equation}\label{eq:geometric}
			|R_n|^{\a -1}\leq \sum _{i=1}^{\8 }\| A_1...A_{i-1}\| ^{\a -1} |B_i| ^{\a -1}.
		\end{equation}
		Since $k$ is log-convex (see Prop. \ref{prop:transferoperators}), there is $\rho <1 $ such that $\E \| \Pi_{i-1}\| ^{\a -1}\leq c\rho ^n$ and so 
		the right hand side of \ref{eq:geometric} is integrable. If $\a >2$ then by the Minkowski inequality we have 
		\begin{align*}
			\left ( \E \left [\sum _{i=1}^{\8 }\|\Pi_{i-1}\| |B_i| \right ]^{\a -1}\right )^{1\slash (\a -1)}&~\le~\sum _{i=1}^{\8 }\left (\E \| \Pi_{i-1}\|^{\a -1} \E |B_i| ^{\a -1} \right )^{1\slash (\a -1)}\\
			&~\le~ \sum _{i=1}^{\8 }\rho ^{(i-1)\slash (\a -1)}\left ( \E |B_1| ^{\a -1} \right )^{1\slash (\a -1)}<\8 .
		\end{align*}
		This shows that we may use dominated convergence in \eqref{eq:bnbn1} to infer that $\lim_{n \to \infty} \frac{b_n}{n}=C.$ 
		
		\medskip

		\noindent \textbf{Step 2}: To obtain the announced results for $\E |R_n|^\alpha$ (without the eigenfunction $r_\alpha$), observe that 
		\begin{equation*}
			b_n\leq c\E |R_n|^{\a}.
		\end{equation*}
		From this, the lower bound follows. 
		We shall prove that the inverse inequality also holds i.e. there is $c_1$ such that
		$c_1\E |R_n|^{\a}\leq  \int _{S } \E \abs{\skalar{y,R_n}}^{\a} \nu ^*_{\a}(dy)$.
		Indeed, the function
		\begin{equation*}
			S\ni x \mapsto  \int_S \abs{\skalar{y,x}}^{\a} \nu ^*_{\a}(dy)
		\end{equation*}
		is continuous and it attains its minimum $c_1$ on $S$. Suppose that $c_1=0$. Then there is $x_1\in S$ such that
		$\abs{\skalar{y,x_1}}^{\a}=0$ on the support of $\nu ^*_{\a}$ i.e. $\mathrm{supp}\ \nu ^*_{\a}$ is contained in a hyperplane orthogonal to $x_1$ which is impossible by Prop. \ref{prop:transferoperators}. Therefore,  
		\begin{equation*}
			\int _{S } \abs{\skalar{y,R_n|R_n|^{-1}}}^{\a} \nu ^*_{\a}(dy)\geq c_1, \quad \mbox{if} \ R_n\neq 0
		\end{equation*}
		and thus $\E \int _{S } \abs{\skalar{y,R_n}}^{\a} \nu ^*_{\a}(dy)\geq c_1 \E |R_n|^{\a}.$	
	\end{proof}
	
	We continue with the proof of Theorem \ref{thm:tailsRn} that provides an upper bound for the tails of finite iterations $R_n$.

	\begin{proof}[Proof of Theorem \ref{thm:tailsRn}]
		Using a union bound, the summability of $\frac{1}{k^2}$ as well as Markov's inequality, we have
		\begin{align*}
			& \P(R_n > t) ~\le~ \sum_{m=1}^n \P\left( \norm{\Pi_{m-1}}|B_m| > \frac{6t}{\pi^2 m^2} \right) \\
			\le& \sum_{m=1}^n \E \norm{\Pi_{m-1}}^{\a+\eps} \E |B_m|^{\a+\eps} \cdot \frac{(\pi^2 m^2)^{\a +\eps}}{(6t)^{\a +\eps}} 
			\le~ \sum_{m=1}^n C \cdot k(\a +\eps)^{m-1} \E |B_m|^{\a+\eps} \frac{(\pi^2 m^2)^{\a +\eps}}{(6t)^{\a +\eps}} \\
			~\le& n C \cdot k(\a +\eps)^{n-1} \E |B_1|^{\a+\eps} \frac{(\pi^2 n^2)^{\a +\eps}}{6^{\a +\eps}} \cdot t^{-(\a+\eps)} ~=:~ C_n t^{-(\a+\eps)}
		\end{align*}
		where we have further used that $\E \norm{\Pi_m}^s \le C_sk(s)^m$ for each $m \in \N$ and $k(\alpha+\epsilon)>1$ due to the convexity of $k$.
	\end{proof}

			\section{Evaluation of tail index and stationary measure: The proofs of Theorems \ref{lem:uniform measure} and \ref{th:behavioralpha} }\label{sect:tailindex}
			
			\subsection{Identification and Differentiability of $k(s)=\E | (I-\xi H)e_1| ^s$}
			
			We start by proving Theorem \ref{lem:uniform measure} that identifies $\nus$ as the uniform measure on $S$ and provides a tractable formula for $k(s)$ under the assumption of rotational invariance.

			\begin{proof}[Proof of Theorem \ref{lem:uniform measure}]
				Let $\xi = \eta \slash b$. For $y \in S$, let $\od$ be an orthogonal matrix such that $y=\od e_1$. We have for any $f \in \Cf(S)$,
				\begin{align}
					&P^sf(y)= P^s f(ke_1)
					=\E f\left ((I-\xi H)\cdot (\od e_1)\right )| (I-\xi H)(\od e_1)| ^s \notag\\
					=&\E f\left (\od(I-\xi \od^{-1}H\od)\cdot e_1\right )| \od(I-\xi \od^{-1}H\od)e_1| ^s \notag \\
					=&\E f\left (\od(I-\xi H)\cdot e_1\right )| (I-\xi H)e_1| ^s \label{eq:eigenfunction_e1}
				\end{align}
				because $\od^{-1}H\od$ has the same law as $H$ and $\od$ preserves the norm. Using this for $f \in \Cf(S)$ with $\norm{f}=1$, we obtain
				$$ |P^s f(y)| \le \E \norm{f} | (I-\xi H)e_1| ^s \le \E | (I-\xi H)e_1| ^s.$$
				In particular, the spectral radius of $P^s$ is bounded by $\E \abs{(I-\xi H)e_1}^s$, thus by Prop. \ref{prop:transferoperators}, $k(s)\le \E \abs{(I-\xi H)e_1}^s$. Using the representation \eqref{eq:uniform.measure} of the uniform measure on $\S$ together with \eqref{eq:eigenfunction_e1}, we obtain further that
				\begin{align*}
					\int _SP^sf(y)\  \s (dy) =& \int_{O(d)} P^s f(\od e_1) \ d\od
					=&\int _{O(d)}\E f\left (\od (I-\xi H)\cdot e_1\right )| (I-\xi H)e_1| ^s\ d\od 
				\end{align*}

				Now we use Fubini and observe that for a fixed $H$
				\begin{equation*}
					\int _{O(d)} f\left (\od (I-\xi H)\cdot e_1\right )\ d\od =\int _{O(d)} f\left (\od \cdot ((I-\xi H)\cdot e_1)\right )\ d\od= \int _{O(d)} f\left (\od\cdot e_1\right )\ d\od .
				\end{equation*} 
				Therefore, 
				\begin{equation*}
					\int P^sf(y)\  \s (dy) =\left(\int _Sf(y)\ \s (dy)\right )\E | (I-\xi H)e_1| ^s , 
				\end{equation*}
				hence $\sigma$ is an eigenmeasure of $P^s$ with eigenvalue $\E | (I-\xi H)e_1| ^s$. Since the latter one is also the upper bound on the spectral radius and on $k(s)$, we infer that we have indeed found the dominating eigenvalue und the corresponding eigenmeasure. Thus, due to the uniqueness results of Prop. \ref{prop:transferoperators}, $\nu _s=\s $ and $k(s)= \E | (I-\xi H)e_1| ^s$.
				
				The formula for the Lyapunov exponent then follows from Prop. \ref{prop:diffk}.
			\end{proof}

		Propositions \ref{prop:transferoperators} and \ref{prop:diffk} provide general results about the differentiability of $k$ as a function of $s$. We now focus on the particular instance of \eqref{rotinv}, where
		$$k(s)=\E | (I-\xi H)e_1| ^s=:h(\xi,s),$$ and are going to prove that $h(\xi,s)$ is $\Cf^1$.
		
		\begin{lemma}\label{lem:partialdiffks}
			Assume \eqref{rotinv} and that $\E |\langle He_2,e_1\rangle |^{-\delta}<\infty$ for some $\delta>0$. Then
			\begin{equation*}
				\frac{\partial h}{\partial s}= \E | (I-\xi H)e_1| ^s\log | (I-\xi H)e_1| 
			\end{equation*}
			exists and  is continous on $(0,\8 )\times [0,s_0)$ with
			\begin{equation}\label{eq:diff0}
				\frac{\partial h}{\partial s}(\xi ,0)= \E \log | (I-\xi H)e_1| .
			\end{equation}
		\end{lemma}
		
		\begin{proof}
			Fix $\xi>0$ and let $s \in [0,s_0)$. By the mean value theorem there is $\theta \in (0,1)$ such that
			\begin{align*}
				\frac{h(\xi,s+v)-h(\xi,s)}{v} ~&=~ \frac{1}{v} \E \Big[ \abs{(I-\xi H)e_1}^{s+v} - \abs{(I-\xi H)e_1}^s \Big] \\
				~&=~ \E \Big[ \abs{(I-\xi H)e_1}^{s+\theta v} \log \abs{(I-\xi H)e_1} \Big]
			\end{align*}
			Thus, the asserted differentiability follows by dominated convergence once we can provide  a bound for the right hand side that is uniform in $\theta$. 
			
			Observe that for any small $\varepsilon >0$, it holds $|\log x| < |x|^\varepsilon$ for $x>1$ and $|\log x| \le |x|^{-\varepsilon}$ for $0 < x <1$, hence
			\begin{align*}
				& \E \abs{(I-\xi H)e_1} ^s\log \abs{(I-\xi H)e_1} \\
				\le & \E \abs{(I-\xi H)e_1} ^{s+\varepsilon} \mathds{1}_{\{\abs{(I-\xi H)e_1} \ge 1\}} + \E \abs{(I-\xi H)e_1} ^{s-\varepsilon} \mathds{1}_{\{\abs{(I-\xi H)e_1} < 1\}} =:T_1+ T_2.
			\end{align*}

			The first term $T_1$ is finite as long as $s+\varepsilon<s_0$. Considering $T_2$, we only need to consider the case where $s-\varepsilon <0$. We use that
			\begin{equation*}
				\abs{(I-\xi H)e_1}= \left ( (1-\xi \langle He_1,e_1\rangle )^2+\xi ^2\sum _{j=2}^d \langle He_j,e_1\rangle )^2 \right) ^{1\slash 2}\geq \xi |\langle He_j,e_1\rangle | .
			\end{equation*}
			Therefore, if $|\langle He_j,e_1\rangle |^{-\varepsilon}$ is integrable for all $j$, then $T_2$ is finite for any $\xi $. Observe that by \eqref{rotinv}, for any function $f$ and $\od \in O(d)$ chosen such that $\od e_2=e_j$, $\od e_1=e_i1$ ($j\neq i$)
			\begin{equation*}
				\E f(\langle He_j,e_i\rangle)= \E f(\langle H\od e_2,\od e_1\rangle )=\E f(\langle \od^TH\od e_2,e_1\rangle )=
				\E f(\langle He_2,e_1\rangle ).
			\end{equation*}
			Therefore, integrability of a matrix coefficient $|\langle He_j,e_i\rangle |^{\varepsilon}$, $j\neq i$ is equivalent
			to integrability of $|\langle He_2,e_1\rangle |^{-\varepsilon}$.
		\end{proof}

			\medskip
			Existence of $\frac{\partial h}{\partial \xi}$ is a little bit more complicated, in particular for $1<s<2$ and it is not well clarified in \cite{Guerbuezbalaban2021supp}. We have the following lemma. 
			\begin{lemma}\label{lem:diff}
				Suppose $s_0> 2$. Then 
				$\frac{\partial h}{\partial \xi}$ exists and is continuous on  $(0,\8 )\times [2,s_0) $ 
				with 
				\begin{equation}\label{eq:difformula0}
					\frac{\partial h}{\partial \xi}(0,s)=-s\E \langle He_1,e_1\rangle .
				\end{equation}
				If in addition $\E |\langle He_2,e_1\rangle |^{-\delta}<\infty$  holds for every $0<\delta <1$, 
				then $\frac{\partial h}{\partial \xi}$ exists and  is continuous on $(0,\8 )\times (1,s_0)$ .  
			\end{lemma}

			\begin{proof}[Proof of Lemma \ref{lem:diff}]
				Let
				\begin{equation*}
					g(\xi )= \langle (I-\xi H)e_1, (I-\xi H)e_1\rangle . 
				\end{equation*}
				By the mean value Theorem, for some $\theta \in (0,1)$,
				\begin{align*}
					\frac{h(\xi +v,s)-h(\xi ,s)}{v}=&\frac{1}{v}\E \left (g(\xi +v)^{s\slash 2}-g(\xi )^{s\slash 2}\right )  \\
					=&\frac{s}{2}\E g(\xi +\theta v)^{s\slash 2-1}g'(\xi +\theta v).
				\end{align*}
				Suppose we were able to dominate $g(\xi +\theta v)^{s\slash 2-1}g'(\xi +\theta v)$ independently of $\theta v$ by an integrable function. Then by dominated convergence
				\begin{equation}\label{eq:difformula}
					\frac{\partial h}{\partial \xi}=s\E | (I-\xi H)e_1| ^{s-2} \left( -\langle He_1, e_1\rangle +\xi \langle He_1, He_1\rangle \right). 
				\end{equation}
				\textbf{Case $s\ge2$}: The right hand side is dominated by $C(\xi)\| H\| ^s$ and so it is finite. $C(\xi)\| H\| ^s$ is also in the right domination in the Lebesgue convergence theorem. In particular, \eqref{eq:difformula0} holds.
				\\
				\textbf{Case $1 <s<2$}: Observe that
				\begin{equation*}
					|g'(\xi +\theta v)|\leq C(\xi)(1+\| H\|^2)
				\end{equation*}
				and
				\begin{align*}
					g(\xi +\theta v)^{s\slash 2-1}=&| (I-(\xi +\theta v)H)e_1| ^{s-2}\\
					=& \left ( (1-(\xi +\theta v) \langle He_1,e_1\rangle )^2+(\xi +\theta v) ^2\sum _{j=2}^d \langle He_j,e_1\rangle )^2 \right) ^{(s-2)\slash 2}
				\end{align*}
				with $-1<s-2<0$.
				As before, 
				\begin{equation*}
					g(\xi +\theta h)^{s\slash 2-1}\leq (\xi\slash 2) ^{s-2} |\langle He_j,e_1\rangle |^{s-2}
				\end{equation*}
				for sufficiently small $\theta v$. 	Similar to the proof of Lemma \ref{lem:partialdiffks}, integrability of $|\langle He_j,e_i\rangle |^{-\delta}$ for some positive $\delta $ is equivalent to integrability of $|\langle He_2,e_1\rangle |^{-\delta}$.
				Therefore, we may apply the Lebesgue convergence and we obtain \eqref{eq:difformula}  for $\xi >0$.
			\end{proof}
			
			\begin{remark}
				It follows from the proof that if  $\E |\langle He_2,e_1\rangle |^{-\delta} <\8$ for $\delta \leq \delta _0<1$ then $h \in \Cf^1 \left ( (0,\8 )\times (2-\delta _0,s_0) \right )$
			\end{remark}	
			
			\subsection{The shape of the tail index function $\a (\xi)$: Proof of Theorem \ref{th:behavioralpha} }
			Now using differentiability of $h$ we study dependence of the tail index on $\xi $.

			\begin{proof}[Proof of Theorem \ref{th:behavioralpha}] Using that $s_0=\infty$ and the assumption that $\supp(H)$ is unbounded, a simple calculation shows that for every fixed $\xi \in U$, $\lim_{s \to \infty} h(\xi,s)=\infty$.
				 By Prop. \ref{prop:transferoperators}, for fixed $\xi$, the function $s\mapsto h(\xi,s)$ is strictly convex, hence for every $\xi \in U$ there is a unique $\alpha(\xi)>0$ with $h(\xi, \alpha(\xi))=1$. 
				
				\noindent \textbf{Step 1}:
				Observe, that for every $s>1$, $h(\cdot , s)$ is a strictly convex function of $\xi $. Moreover, $h(0,s)=1$ and, by \eqref{eq:difformula0}, \eqref{eq:Hpos2}, for $s\geq 2$, $\frac{\partial h}{\partial \xi}(0,s)<0$. Hence for each  $s$ there is a unique $\xi >0$ such that $h(\xi ,s)=1$.  Therefore the set 
				$$\tilde U=\{ \xi : \exists\, \a >1 \text{ s.t. }  h(\xi ,\a)=1 \}$$ is not empty, and the mapping $\xi \to \alpha(\xi)$ is one-to-one on $\tilde{U}$.

				Let us prove that $\tilde U$ is open. Let $\xi _*\in U$ and $h(\xi _*, \a_*)=1$. Given $\xi $, $h(\xi , \cdot)$ is a strictly convex $C^1$ function of $s$ and $h(\xi , 0)=1$. Therefore,
				$\frac{\partial h}{\partial s}(\xi _*, \a_*)>0$ and the same holds in a neighbourhood of $(\xi _*, \a_*)$. Moreover, there is $\eps _0>0$ such that $h(\xi _*, \a_*-\eps _0 )<1$,  $h(\xi _*, \a_*+\eps _0 )>1$ and $\a_*-\eps _0 >1$. Let $\eta _0>0$ be such that for $0<\eta <\eta _0$, $h(\xi _*\pm \eta , \a_*-\eps _0)<1$, $h(\xi _*\pm \eta , \a_*+\eps _0)>1$ and  $\frac{\partial h}{\partial s}(\xi _*\pm \eta, \a_*\pm \eps )>0$ for $0\leq \eta <\eta _0$, $0\leq \eps <\eps _0$. 
				This proves that for such $\eta $, we have $\alpha(\xi _*\pm \eta )>1$ and so $(\xi _*-\eta _0, \xi _*+\eta _0)\subset \tilde U$.   \\
				\textbf{Step 2}: Now we shall prove that 
				\begin{equation}\label{eq:rightend}
					\tilde U=(0,\xi _1)\quad \mbox{and}\quad \a(\xi _1)=1.	
				\end{equation}
				Being an open set, $\tilde U$ could possibly be a sum of disjoint open intervals. Let $I:=(\xi ', \xi '' )$ be an interval with $I \subset U$, and every open connected interval containing $I$ is not a subset of $U$, {\em i.e.}, $I$ is maximal. Let $\xi _*\in I$. 
				Then by strict convexity of both $h(\xi , \cdot)$ and 
				$h( \cdot ,s)$, we have $\frac{\partial h}{\partial \xi}(\xi _*, \a_*)$, $\frac{\partial h}{\partial s}(\xi _*, \a_*)>0$. Hence by the implicit function theorem, $\a(\xi)$ is a $\Cf^1$ function on $(\xi ', \xi '' )$ and 
				$$ 				\a'(\xi)<0 \quad \mbox{for}\quad \xi \in (\xi ',\xi '').$$
				This proves the first part of assertion \ref{k3} (except for the identification $I=(0,\xi)$.) \\
				\textbf{Step 3}:
				Both $\lim _{\xi \to (\xi '')^-}\a(\xi )$ and $\lim _{\xi \to (\xi ')^+}\a(\xi )$ exist with the latter being possibly unbounded. Suppose  $\a_1= \lim _{\xi \to (\xi '')^-}\a(\xi )>1$. Then $h(\xi '', \a_1)= \lim _{\xi \to (\xi '')^-}h(\xi ,\a(\xi ))=1$ which shows that $\xi ''\in \tilde U$ and $(\xi ', \xi '' )$ is not maximal. This proves the existence of $\xi _1$.  
				
				Suppose that $\a_1= \lim _{\xi \to (\xi ')^+}\a(\xi )$ is finite. Then $h(\xi ', \a_1)= \lim _{\xi \to (\xi ')^+}h(\xi ,\a(\xi ))=1$ which shows that either $\xi '=0$ or $\xi '\in U$ and thus $(\xi ', \xi '' )$  is not maximal. 
				
				Therefore, $ \lim _{\xi \to (\xi ')^+}\a(\xi )=\8 $ or $\xi '=0$. Let us prove that in fact $\xi '=0$ and $ \lim _{\xi \to 0^+}\a(\xi )=\8 $. Suppose that $\xi '>0$. The range of $\a(\xi)$ on $(\xi ', \xi '')$ is $(1,\8 )$ and so by strict convexity of $h(\cdot ,s')$ for given $s'>1$ 
				there is $0<\xi _*< \xi '$ such that $h(\xi _* ,s')<1$. On the other hand, $\lim _{s\to \8}h(\xi _*,s)=\8$ and so there is $s_*>s' $ such that $h(\xi _*,s_*)=1$. But for $s^*$ 
				there is also $\xi > \xi '$ such that $h(\xi ,s_*)=1$ which is not possible by strict convexity of $h(\cdot , s_*)$. This proves \eqref{eq:rightend} and hence assertion \ref{k2}.
				
				Similarly, $ \lim _{\xi \to 0^+}\a(\xi ) $ cannot be finite. Indeed, suppose that $\a_1= \lim _{\xi \to 0^+}\a(\xi )<\8 $.
				Then by \eqref{eq:Hpos2}, for every $s>\max (\a_1,2)$ there is $\xi $ close to $0$ such that $h(\xi , s)<1$. At the same time 
				there is $s<\a_1$ such that $h(\xi , s)=1$. Since $k(\xi , \cdot)$ is strictly convex, we get a contradiction. 
				
				Finally by strict convexity of $h(\cdot , s)$ for $s>1$, $\a(\xi )$ must be strictly smaller than 1 as far as $\xi >\xi _1$. 
			\end{proof} 
			
			\begin{remark}
				By \eqref{eq:difformula0}, the assumption $\E \langle He_1,e_1\rangle < 0$ is equivalent to $U$ being not empty and
				$\sup \{ s(\xi): \xi \in U\} >2$. Suppose however we know only that $\sup \{ s(\xi): \xi \in U\}>1$.
				Then the same argument as in the proof allows us to conclude that  $s$ is strictly decreasing on an interval $(0, \xi _1)$, $h(\xi _1,1)=1$ but this time $ \lim _{\xi \to 0 ^+}s(\xi )$ 
				may possibly be finite. 
			\end{remark}

			\section{Specific models}\label{sect:GaussianModel}
			
			\subsection{Hierarchy of assumptions}

			\begin{proof}[Proof of Lemma \ref{lem:hierarchyAssump}] We start by proving the first implication.
				By assumption, $a_i$ have a $d$-dimensional standard Gaussian distribution, hence $\od a_i$ has the same law as $a_i$ for any $\od \in O(d)$. Thus each $H_i$ is invariant under rotations. Further,
				$$ \od A \od^\t ~=~ \od \big( \Id - \frac{\eta}{b}\sum_{i=1}^b H_i\big) \od^\t ~=~ \Id - \frac{\eta}{b} \sum_{i=1}^b \od H_i\od^{\top} ~\eqdist~ \Id - \frac{\eta}{b} \sum_{i=1}^b H_i$$
				where we have used that $H_i$ are iid and that $\od$ is deterministic to replace each summand $\od H_i\od^\top$ by the random variable $H_i$ which is the same in law only. This proves the first part of \eqref{rotinv}.
				In order to show that $G_A$ contains a proximal matrix, note that $\supp(a_i) = \R^d$. Let $v \in \R^d$ be such that $\abs{v}^2 > \frac{2b}{\eta}$. Then $(v,0,0,\dots,0) \in \mathrm{supp}\big((a_1, \dots, a_b)\big)$, hence
				$ \Id - \frac{\eta}{b} v v^\top ~\in~G_A$ and $v$ is an eigenvector with corresponding eigenvalue $1-\frac{\eta}{b} \abs{v}^2 < -1$. Its eigenspace is one-dimensional, any $y \in \R^d$ that is orthogonal to $v$ is just mapped to itself (eigenvalue 1).

				Considering the second implication, it is proved in \cite[Proposition IV.2.5]{Bougerol1985} (applied with $p=1$) that \eqref{rotinv} implies condition (i-p). By \cite[Proposition 2.14]{Guivarch2016}, there is either no $G_A$-invariant cone, or there is a minimal $G_A$-invariant subset $M$ of $S$ that generates a proper cone $G_A$-invariant cone $C$. Using the rotation invariance, we can deduce from $G_A \cdot M \subset M\subset C$ that $\od^{-1}G_A \od \cdot M \subset M\subset C$ and hence $G_A \od \cdot M \subset \od M \subset \od C$ for all $\od \in O(d)$. Consequently, $G_A \cdot (\od M \cap M) \subset (\od M \cap M)$, and since $O(d)$ acts transitively on $S$, there is $\od$ such that $M \nsubseteq \od M$. Hence $\od M \cap M$ is $G_A$-invariant and strictly smaller than $M$; which contradicts the minimality of $M$.
			\end{proof}

			\subsection{The Gaussian model}

			\begin{proof}[Proof of Lemma \ref{lem:densityGauss}]
				Let us write $a_{i,k}$, $1 \le i \le b$, $1 \le k \le d$ for the $k$-th component of the $i$-th vector $a_i$, and note that $(a_{i,k})_{i,k}$ are i.i.d.\ standard Gaussian random variables.\\
				\textbf{Step 1}: Recalling that
				
				$$ a_i a_i^\top ~=~ \begin{pmatrix}
					a_{i,1}^2 & a_{i,1}a_{i,2} & a_{i,1}a_{i,3} & \dots & a_{i,1}a_{i,d} \\[.2cm]
					& a_{i,2}^2 & a_{i,2} a_{i,3} & \dots & a_{i,2}a_{i,d} \\[.2cm]
					&  & a_{i,3}^2 & \dots & a_{i,3}a_{i,d} \\[.2cm]
					& &  & \ddots &\vdots \\[.2cm]
					& & & & a_{i,d}^2
				\end{pmatrix},$$
				we obtain for $H=\sum_{i=1}^b a_i a_i^\top$ 
				that 
				$$ H_{\ell, \ell} = \sum_{i=1}^b a_{i, \ell}^2, \qquad H_{\ell,k } = \sum_{i=1}^b a_{i,\ell} a_{i,k} $$
				Introducing the $b$-dimensional vectors $x_k=(x_{k,i})_{i=1}^b$, $1 \le k \le d$, by  setting $x_{k,i}=a_{i,k}$, we see that
				$$ H_{\ell, \ell} = \abs{x_{\ell}}^2, \qquad H_{\ell,k } = \sum_{i=1}^b \skalar{x_\ell, x_k} $$
				Note that $x_k$, $1 \le k \le d$ are i.i.d.\ $b$-dimensional standard Gaussian vectors. 	
				Hence, $Z_\ell^2:=H_{\ell,\ell}$,$ 1 \le \ell \le d$ are independent and identically distributed random variables with a $\chi^2(b)$-distribution, and 
				$Y_\ell := \frac{x_\ell}{\abs{x_\ell}}$ has  the uniform distribution on $S^{b-1}$, due to the radial symmetry of the $b$-dimensional Gaussian distribution. Moreover, $Y_{\ell}$ is independent of $H_\ell$, so that we have that all the random variables $Z_{1}, \dots, Z_{d}, Y_{1}, \dots, Y_d$ are independent, and
				$$ H ~=~ \begin{pmatrix}
					Z_1^2 & Z_1Z_2 \skalar{Y_1,Y_2} &Z_1Z_3 \skalar{Y_1,Y_3} & \dots & Z_1Z_d \skalar{Y_1,Y_d} \\[.2cm]
					& Z_2^2 & Z_2Z_3 \skalar{Y_2,Y_3} & \dots & Z_2Z_d \skalar{Y_2,Y_d} \\[.2cm]
					&  & Z_3^2 & \dots & Z_3Z_d \skalar{Y_3,Y_d} \\[.2cm]
					& &  & \ddots &\vdots \\[.2cm]
					& & & & Z_d^2
				\end{pmatrix},$$
				The density of each $Z_\ell^2$ is given by the $\chi^2(b)$-density: $$f_Z(z)=\frac{1}{2^{b/2}\Gamma(b/2)} z^{b/2-1} e^{-z/2}, \qquad z>0.$$
				\textbf{Step 2}:  We may then condition on $Z_1^2, \dots, Z_d^2$, and analyze the joint density of the $d(d-1)$ inner products $U_{\ell ,k}:=\skalar{Y_\ell, Y_k}$, $1 \le \ell < k \le d$. This has been done in the proof of \cite[Theorem 4]{Stam1982}. Since $Y_\ell$, $1 \le \ell \le d$ are independent random vectors that have the uniform distribution on $S^{b-1}$, the assumptions of that theorem are satsified. 
				In the proof, a recursive formula for the joint density $p(u_{1,2}, \dots, u_{d-1,d})$ of  $\sqrt{b}\cdot U_{\ell ,k}$, $1 \le \ell < k \le d$ is obtained, which is continuous and positive (due to the assumption $b>d+1$) on $(-\sqrt{b},\sqrt{b})^{d(d-1)/2}$ and zero outside. 
				For ease of the reference, we quote the formula obtained there. Write $p^{(b)}_d$ for the joint density of the inner products of $d$ random unit vectors in $b$-dimensional space.
				Then for $d=2$, $b>3$,
				$$ p^{(b)}_2(u) = \frac{\pi^{-1/2} \Gamma(\tfrac{b}2)}{\Gamma(\tfrac{b-1}{2})} (1-u^2)^{\tfrac{b-3}{2}}$$
				and recursively, for $b>d+1$
				\begin{align*}
					& p_{d}^{(b)}(u_{1,2}, \dots, u_{d-1,d}) \\ =& (b \pi)^{-\tfrac12 (d-1)} \left(\frac{\Gamma(\tfrac{b}{2})}{\Gamma(\tfrac{b-1}{2})}\right)^{d-1} \prod_{i=2}^d \Big(1- \frac{u_{1,i}^2}{b}\Big)^{\tfrac12 (b-d-1)} p_{d-1}^{(b-1)}(w_{2,3}, \dots, w_{r-1,r})
				\end{align*}
				where
				$$ w_{\ell,k} = \frac{u_{i,j}-b^{-\tfrac12}u_{1,i}u_{1,j}}{(1-b^{-1}u_{1,i}^2)^{1/2}(1-b^{-1}u_{1,j}^2)^{1/2}} \quad 2 \le \ell < k \le d. $$
				We obtain that $H_{\ell, \ell}=Z_\ell^2$ and $H_{\ell,k}=Z_\ell Z_k U_{\ell,k}$, where the vector $$(Z_1^2, \dots, Z_d^2, \sqrt{b}U_{1,2}, \dots \sqrt{b}U_{d-1,d})$$ has a density $g$ with respect to Lebesgue measure on $\R^{d(d+1)/2}$ given by
				$$ g(z_1, \dots, z_d, u_{1,2}, \dots, u_{d-1,d}) ~=~ C p^{(b)}_d(u_{1,2}, \dots, u_{d-1,d}) \prod_{\ell=1}^d z_\ell^{b/2-1} e^{-z_\ell/2}. $$

				\textbf{Step 3}: A density of $(Z_iZ_j U_{i,j})$ is then obtained by first considering the conditional density of $\sqrt{z_iz_j} U_{i,j}$ for fixed $Z_i^2=z_i$, $Z_j^2=z_j$, which is obtained from $p(\dots, u_{i,j}, \dots)$ by the change of measure $\widetilde{u_{i,j}}=\frac{\sqrt{z_i z_j}}{\sqrt{b}} u_{i,j}$, hence $du_{i,j}$ becomes $\frac{\sqrt{b}}{\sqrt{z_i z_j}} du_{i,j}.$
				We thus finally obtain that the entries $(H_{\ell,k})_{1 \le \ell \le k \le d}$ have the density
				$$ f(y) =C \cdot b^{d(d-1)/4}  \cdot p^{(b)}_d(\frac{\sqrt{b}}{\sqrt{y_{1,1}y_{2,2}}}y_{1,2}, \dots, \frac{\sqrt{b}}{\sqrt{y_{d-1,d-1}y_{d,d}}}y_{1,2}y_{d-1,d}) \prod_{\ell=1}^d y_{\ell,\ell}^{(b-d+1)/2-1} e^{-y_{\ell,\ell}/2} $$
				where $y=(y_{\ell,k})_{1 \le \ell \le k \le d}$. Due to the Cauchy-Schwartz-inequality, the off-diagonal entries of a positive semi-definite matrix (as H is), are always bounded by the square root of the product of the corresponding diagonal entries, i.e., $y_{\ell,k} \le \sqrt{y_{\ell,\ell} y_{k,k}}$. This gives that 
				$$ \abs{y}^2 \le (\sum_{\ell=1}^d y_{\ell,\ell}^2)^2$$
				Since $p_d^{(b)}$ is bounded, we finally obtain (due to the exponential term) that 
				$$ f(y)\leq C\left ( 1+\abs{y}^2\right )^{-D}$$
				for any $D>0$.
			\end{proof}

			\begin{proof}[Proof of Corollary \ref{cor:Rank1Gauss}]
				We start by noting that $$ Ax+B=x ~\Leftrightarrow~ x=(\Id-A)^{-1}B=\Big(\sum_{i=1}^b a_i a_i^\top\Big)^{-1} \sum_{i=1}^b a_i y_i $$
				Since $y_i$ are independent of $a_i$, this identity cannot hold $\P$-a.s. for a fixed $x \in \R^d$, hence $P(Ax+B=x)<1$ for all $x \in \R^d$.
				By Lemma \ref{lem:hierarchyAssump}, \eqref{Rank1Gauss} implies \eqref{rotinv} and \ipnc. By Lemma \ref{lem:densityGauss}, condition \eqref{decay} is fulfilled and hence Theorem \ref{th:integrab} applies and asserts that the moment conditions of all theorems are satisfied.
				Finally, $H=\frac{\eta}{b}\sum_{i=1}^b a_ia_i^\top$ is positive semi-definite and nontrivial, hence $\E \skalar{He_1, e_1}>0$ holds as well. Thus all assumptions of the theorems are satisfied.
				
				The dependence of $\alpha$ on $\eta$ is a direct consequence of Theorem \ref{lem:densityGauss}. Note that the dependence on $b$ does not follow from that theorem, since changing $b$ also changes the law of $H$. Instead, this follows as in \cite[C.3, part (II)]{Guerbuezbalaban2021supp}. Note that the argument given there remains valid under \eqref{rotinv}.
			\end{proof}

			\section{Integrability of $\| A^{-1}\|^{\delta } $}  \label{sect:momentbounds} 
		
		In this section, we study integrability of $N(A)$, in particular of moments of $\norm{A^{-1}}$ or - similarly - negative moments of $\det(A)$. The approach taken here is applicable to general classes of matrices and relies on the following general result about integrability of polynomials.

		\begin{theorem}\label{th:integrability} 
			Let $W(x_1,...,x_m)$ be a polynomial on $\R ^m$ such that the maximal degree of $x_i$ in $W$ is at most 2 for every $i$ and $W$ is not identically zero. Let $f$ be a nonnegative function of $\R ^m$ such that $f(x)\leq C(1+| x| ^2)^{-D}$ for $D>m\slash 2$. Then for every $\eps <1\slash 2$
			\begin{equation}\label{eq:W}
				\int _{\R ^m}|W(x)|^{-\eps }f(x)\ dx<\8.
			\end{equation} 
			If the maximal degree of $x_i$ in $W$ is at most 1 for all $i$, then \eqref{eq:W} holds for every $\epsilon<1$.
		\end{theorem}

		\begin{proof}[Proof of Theorem \ref{th:integrability}]
			We use induction with respect to the number $m$ of variables. We only provide the induction step, since the proof for $m=1$ is along the same lines (omitting dependence on $\bar{x}$ in all calculations).
			To wit, we are going to prove that
			\begin{equation}\label{eq:WP}
				\int_{\R^m} |W(x)|^{-\epsilon} f(x)\ dx \le C' \int_{\R^{m-1}} |P(\bar{x})|^{-\epsilon} \bar{f}(\bar x) \ d \bar x 
			\end{equation}
			for a constant $C'$, a polynomial $P$ with the same properties as $W$, and a nonnegative function $\bar{f}$ that satisfies $\bar{f}(\bar x) \le C'' (1 + \abs{x}^2){-D'}$ for $D'=D-\frac12>{(m-1)/2}$.
			
			\textbf{Step 1}: Fix $\epsilon \in (0,\tfrac12)$. Let $x=(x_1,\bar x)\in \R ^m$ for a variable $x_1 \in \R$ that appears in $W$ [in case that $W$ does not depend on $x_1$, then \eqref{eq:Winterval} holds automatically with $C(\epsilon)=1$ and $P(\bar x)=W(\bar x)$]. Then we may write $W(x)=P(\bar x)x_1^2+P_1(\bar x)x_1+P_0(\bar x)$ or $W(x)=P(\bar x)x_1+P_0(\bar x)$ with $P$ being not identically zero and satisfying the induction assumption. We will prove \eqref{eq:WP} with for this $P$. More precisely, we will show the inequality on the set $\{ \bar{x} \in \R^{m-1} \, : \, P(\bar x) \neq 0\}$ and use the fact that its complement has $m-1$-Lebesgue measure zero (this is true for any polynomial that is not identically zero).

			We employ Lemma 2.3 in \cite{greenblatt:2005}, which gives that there is a constant $C(\eps)$ such that for every $n \in \N,$
			\begin{equation}\label{eq:Winterval}
				\int _n^{n+1}|W(x_1,\bar x)|^{-\eps}\ dx_1\leq C(\eps )|P(\bar x)|^{-\eps}
			\end{equation}
			as long as $P(\bar x)\neq 0$ -- note that we may exclude the case $P(\bar x)= 0$ by the above discussion.
			For $p>0$ such that $p\eps <1\slash 2$ we have by an application of H\"older's inequality
			\begin{align*}
				\int _{\R }|W(x_1,\bar x)|^{-\eps }f(x_1,\bar x)\ dx_1&=\sum _{n=-\8}^{\8}\int _n^{n+1}|W(x_1,\bar x)|^{-\eps }f(x_1,\bar x)\ dx_1 \\
				&\leq \sum _{n=-\8}^{\8}\left (\int _n^{n+1}|W(x_1,\bar x)|^{-p\eps }\right )^{1\slash p}
				\left (\int _n^{n+1}f(x_1,\bar x)^q\ dx_1\right )^{1\slash q}\\
				&\leq C(\eps p)^{1\slash p}|P(\bar x)|^{-\eps }\sum _{n=-\8}^{\8}
				\left (\int _n^{n+1}f(x_1,\bar x)^q\ dx_1\right )^{1\slash q},  
			\end{align*} 
			with $q=\frac{p}{p-1}$.
			
			\textbf{Step 2}: It remains to prove that
			\begin{equation*}
				\bar f(\bar x) =\sum _{n=-\8}^{\8}
				\left (\int _n^{n+1}f(x_1,\bar x)^q\ dx_1\right )^{1\slash q} \leq C (1+| \bar x| ^2)^{-D_1}, \quad D_1>\frac{m-1}{2}.
			\end{equation*}
			Changing variables $x_1=\sqrt{1+| \bar x| ^2}u$, we have 
			\begin{align*}
				\int _n^{n+1}f(x_1,\bar x)^q\ dx_1 &\leq C\int _n^{n+1}(1+x_1^2+| x| ^2)^{-Dq}\ dx_1\\
				&=C(1+|\bar x| ^2)^{-Dq+1\slash 2}\int _{n\slash \sqrt{1+| \bar x| ^2}}^{(n+1)\slash \sqrt{1+| \bar x| ^2}}(1+u^2)^{-Dq} \ du\\
				&\leq C(1+| \bar x| ^2)^{-Dq}\left (1+n^2\left (1+| \bar x| ^2\right )^{-1}\right )^{-Dq}
			\end{align*}
			Hence
			\begin{align*}
				\bar f(\bar x) &=\sum _{n=-\8}^{\8}\left (\int _n^{n+1}f(x_1,\bar x)^q\ dx_1\right )^{1\slash q}\\
				& \leq C (1+| \bar x| ^2)^{-D}\sum _{n=-\8}^{\8}\left (1+n^2\left (1+| \bar x| ^2\right )^{-1}\right )^{-D}\\
				& \leq C (1+| \bar x| ^2)^{-D}\left (1+\sqrt{1+| \bar x| ^2} \int _{-\8}^{\8}(1+u^2)^{-D} \ du \right )\\
				&\leq C (1+| \bar x| ^2)^{-D+1\slash 2}
			\end{align*}
			and $D-1\slash 2>\frac{m-1}{2}$. This concludes the proof of \eqref{eq:W}.
			
			\textbf{Step 3}: It remains to prove the final assertion of the Theorem. If $W(x)=P(\bar x)x_1+P_0(\bar x)$ the induction step is a simple change of variables and it holds for any $\eps <1$. Indeed, for $\eps p<1$ we have
			\begin{align*}
				& \int _{\R }|W(x_1,\bar x)|^{-\eps }f(x_1,\bar x)\ dx_1 \\
				\leq& \left (\int _{\{x_1: |W(x_1,\bar x)|\leq 1\}}|P(\bar x)x_1+P_0(\bar x)|^{-p\eps }\right )^{1\slash p}
				\left (\int _{\R }f(x_1,\bar x)^q\ dx_1\right )^{1\slash q}
				+\int _{\R }f(x_1,\bar x)\ dx_1\\
				\leq&  C|P(\bar x)|^{-1\slash p }\left (\int _{\R }f(x_1,\bar x)^q\ dx_1\right )^{1\slash q}
				+\int _{\R }f(x_1,\bar x)\ dx_1  
			\end{align*} 
			and from here one can proceed as in Step 2.
		\end{proof}

		\begin{proof}[Proof of Theorem \ref{th:integrab}]
			As discussed before, the proof will rely on several applications of Theorem \ref{th:integrability} together with the assumed decay properties.
			\\
			\textbf{Step 1}: We consider the determinant $\det A=\det(I-\xi H)$ as a polynomial in the (random) entries of $H$. Note that these entries have a joint density due to assumption \eqref{decay} or \eqref{decay1}, respectively. In particular, the determinant is a polynomial in $d(d+1)$ variables (case \eqref{Symm}) or in $db $ variables (case \eqref{Rank1}). The constant term in $\det A$  is 1 so the polynomial $\det A$ is not identically zero. Each summand of the determinant is a product of different entries, hence the maximal degree of any variable in it is less or equal to two -- it can be two because off-diagonal entries appear twice due to symmetry. 
			
			We may thus apply Theorem \ref{th:integrability} with $W=\det A$, observing that the \eqref{decay} and \eqref{decay1} guarantee the required properties of $f$. We thus obtain \eqref{eq:det}.
			\\
			\textbf{Step 2}: First note that
			$$\E (1+\norm{A}^s) N(A)^\epsilon \le \E \norm{A}^\epsilon + \E \norm{A^{-1}}^\epsilon + \E \norm{A}^{s+\epsilon} + \E \norm{A}^s \norm{A^{-1}}^\epsilon. $$
			For the last term, we use H\"older's inequality with $p>1$ chosen such that $ps < s_0$, to obtain
			$$ \E \norm{A}^s \norm{A^{-1}}^\epsilon \le \Big( \E \norm{A}^{sp}\Big)^{\frac1p} \Big( \E \norm{A^{-1}}^{q\epsilon} \Big)^{\frac1q} $$
			with $q=\frac{p}{p-1}$. Hence, if we can show that $\E \norm{A^{-1}}^\delta < \infty$ for some $\delta>0$, then we may choose $\epsilon = \frac{\delta}{q}<\delta$ and obtain finiteness of $\E (1+\norm{A}^s) N(A)^\epsilon$. To prove finiteness of $\E \norm{A^{-1}}^\delta$, we compare the norm with the determinant, as follows. For any $x \in S$,
			\begin{equation*}
				| Ax| \geq |\lambda _1|=\frac{|\det A |}{|\lambda _2\cdots \lambda _d|}, 
			\end{equation*}
			where $\lambda _1\dots \lambda _d$ denote the eigenvalues of $A$, $|\lambda _1|\leq ... \leq |\lambda _d|\le \norm{A}$. Hence
			\begin{equation*}
				\norm{A^{-1}}=\Big(\inf_{x \in \S} \abs{Ax}\Big)^{-1} = \sup_{x \in \S} | Ax| ^{-1} \leq \| A\| ^{d-1}|\det A | ^{-1}. 
			\end{equation*}
			Therefore, using the Cauchy-Schwartz inequality,
			\begin{align*}
				\E \| A^{-1}\| ^{\delta}\leq & \E \| A\| ^{(d-1)\delta}|\det A |^{-\delta}\\
				\leq & \left (\E \| A\| ^{2(d-1)\delta}\right )^{1\slash 2}\left (\E |\det A |^{-2\delta}\right )^{1\slash 2}.
			\end{align*} 
			This gives the desired finiteness of $\E \norm{A^{-1}}^\delta$ with some $\delta < \frac14$, using the finiteness of $\E |\det A|^{-2\delta}$ that was proved in the first step.
			\\
			\textbf{Step 3}: In order to show that $\E |\langle H e_2, e_1\rangle |^{-\delta}<\infty$ for all $\delta \in (0,1)$, we observe that $\langle H e_2, e_1\rangle = x_{12}$ for model \eqref{Symm} or $\langle H e_2, e_1\rangle = \sum _{i=1}^ba_{i,1}a_{i,2}$ for model \eqref{Rank1}, respectively. In particular, it is a polynomial with maximal degree of any variable being one in both cases. Hence, the second part of Theorem \ref{th:integrability} applies and yields the desired result.
		\end{proof}

			\nocite{Gao2015,Guerbuezbalaban2021,Guivarch2016,Guerbuezbalaban2021supp}

\providecommand{\bysame}{\leavevmode\hbox to3em{\hrulefill}\thinspace}
\providecommand{\MR}{\relax\ifhmode\unskip\space\fi MR }
\providecommand{\MRhref}[2]{%
	\href{http://www.ams.org/mathscinet-getitem?mr=#1}{#2}
}
\providecommand{\href}[2]{#2}


\begin{thebibliography}{10}
	
	\bibitem{AM2012}
	Gerold Alsmeyer and Sebastian Mentemeier, \emph{Tail behaviour of stationary
		solutions of random difference equations: the case of regular matrices}, J.
	Difference Equ. Appl. \textbf{18} (2012), no.~8, 1305--1332.
	
	\bibitem{Bougerol1985}
	Philippe Bougerol and Jean Lacroix, \emph{Products of random matrices with
		applications to {S}chr\"{o}dinger operators}, Progress in Probability and
	Statistics, vol.~8, Birkh\"{a}user Boston, Inc., Boston, MA, 1985.
	\MR{886674}
	
	\bibitem{Buraczewski2016}
	Dariusz Buraczewski, Ewa Damek, and Thomas Mikosch, \emph{Stochastic models
		with power-law tails. the equation x=ax+b}, Springer Series in Operations
	Research and Financial Engineering, Springer International Publishing, 2016.
	
	\bibitem{BurDamZen2016}
	Dariusz Buraczewski, Ewa Damek, and Jacek Zienkiewicz, \emph{On the
		{K}esten-{G}oldie constant}, Journal of Difference Equations and Applications
	\textbf{22} (2016), no.~11, 1646--1662. \MR{3590406}
	
	\bibitem{Gao2015}
	Zhiqiang Gao, Yves Guivarc'h, and \'{E}mile Le~Page, \emph{Stable laws and
		spectral gap properties for affine random walks}, Annales de l'Institut Henri
	Poincar\'{e} Probabilit\'{e}s et Statistiques \textbf{51} (2015), no.~1,
	319--348. \MR{3300973}
	
	\bibitem{greenblatt:2005}
	Michael Greenblatt, \emph{Newton polygons and local integrability of negative
		powers of smooth functions in the plane}, Transactions of the American
	Mathematical Society \textbf{358} (2006), no.~2, 657--670. \MR{2177034}
	
	\bibitem{Guivarch2016}
	Yves Guivarc'h and \'{E}mile Le~Page, \emph{Spectral gap properties for linear
		random walks and {P}areto's asymptotics for affine stochastic recursions},
	Annales de l'Institut Henri Poincar\'{e} Probabilit\'{e}s et Statistiques
	\textbf{52} (2016), no.~2, 503--574. \MR{3498000}
	
	\bibitem{Guerbuezbalaban2021}
	Mert G{\"{u}}rb{\"{u}}zbalaban, Umut Simsekli, and Lingjiong Zhu, \emph{The
		heavy-tail phenomenon in {SGD}}, Proceedings of the 38th International
	Conference on Machine Learning, {ICML} 2021, 18-24 July 2021, Virtual Event
	(Marina Meila and Tong Zhang, eds.), Proceedings of Machine Learning
	Research, vol. 139, {PMLR}, 2021.
	
	\bibitem{Guerbuezbalaban2021supp}
	\bysame, \emph{The heavy-tail phenomenon in {SGD}. supplementary document.},
	Proceedings of the 38th International Conference on Machine Learning, {ICML}
	2021, 18-24 July 2021, Virtual Event (Marina Meila and Tong Zhang, eds.),
	Proceedings of Machine Learning Research, vol. 139, {PMLR}, 2021,
	pp.~3964--3975.
	
	\bibitem{Hodgkinson2021}
	Liam Hodgkinson and Michael Mahoney, \emph{Multiplicative noise and heavy tails
		in stochastic optimization}, Proceedings of the 38th International Conference
	on Machine Learning (Marina Meila and Tong Zhang, eds.), Proceedings of
	Machine Learning Research, vol. 139, PMLR, 18--24 Jul 2021, pp.~4262--4274.
	
	\bibitem{Kesten1973}
	Harry Kesten, \emph{{Random difference equations and renewal theory for
			products of random matrices}}, Acta Math. \textbf{131} (1973), 207--248.
	
	\bibitem{ShalevShwartz2014}
	Shai Shalev-Shwartz and Shai Ben-David, \emph{Understanding machine learning:
		from theory to algorithms}, Cambridge University Press, 2014.
	
	\bibitem{Stam1982}
	Aart~J. Stam, \emph{Limit theorems for uniform distributions on spheres in
		high-dimensional {E}uclidean spaces}, Journal of Applied Probability
	\textbf{19} (1982), no.~1, 221--228. \MR{644435}
	
\end{thebibliography}
\end{document}